\documentclass[draftcls,onecolumn,11pt]{IEEEtran}
\usepackage[dvips]{graphicx}
\usepackage{amsmath,amstext,amssymb,epsf}
\usepackage{fullpage,latexsym}
\usepackage{epsfig}
\usepackage{setspace}

\numberwithin{equation}{section} 

\newtheorem{theorem}{\sc Theorem}

\newtheorem{lemma}{\sc Lemma}
\newtheorem{coro}{\sc Corollary}
\newtheorem{req}{\sc Requirement}
\newtheorem{nota}{\sc Notation}
\newtheorem{defin}{\sc Definition}
\newtheorem{rem}{\sc Remark}
\newtheorem{cla}{\sc Claim}
\newtheorem{ex}{\sc Example}

\newenvironment{proof}{\par \sc Proof.\rm}{\hspace*{\fill}$\bullet$\vspace{1ex}}
\newenvironment{example}{\begin{ex}}{\hspace*{\fill}$\diamondsuit$\end{ex}}
\newenvironment{claim}{\begin{cla}}{\end{cla}}

\newenvironment{definition}{\begin{defin}}{\end{defin}}
\newenvironment{remark}{\begin{rem}}{\hspace*{\fill}$\Diamond$\end{rem}}

\renewcommand{\emptyset}{\varnothing}

\begin{document}

\title{Algorithmic Identification of Probabilities}
\author{Paul M.B. Vit\'anyi and Nick Chater
\thanks{Vit\'anyi is with the 
National Research Institute for Mathematics
and Computer Science in the Netherlands (CWI) and
the University of Amsterdam.
Address: CWI, Science Park 123, 
1098 XG, Amsterdam, The Netherlands.
Email: {\tt paulv@cwi.nl}. 

Chater is with the 
Behavioural Science Group.
Address: Warwick Business School, University of Warwick, Coventry, CV4 7AL, UK.
Email: {\tt Nick.Chater@wbs.ac.uk}. Chater was supported by ERC Advanced Grant ``Cognitive and Social Foundations of Rationality.''
}}

\maketitle
\begin{abstract}
The problem is to identify a probability 
associated with a set of natural numbers, 
given an infinite data sequence
of elements from the set. If the given sequence is drawn 
i.i.d. and the probability mass function involved 
(the target) belongs to a computably enumerable (c.e.) or
co-computably enumerable (co-c.e.) set of 
computable probability mass
functions, then there is an algorithm
to almost surely identify the target in the limit.
The technical tool is the strong law of large numbers. 
If the set is finite and the elements of the sequence are dependent 
while the sequence is typical in the sense of Martin-L\"of for at least one
measure belonging to a c.e. or co-c.e. set of computable measures,
then there is an algorithm to identify in the limit
a computable measure for which the sequence is typical (there 
may be more than one such measure). 
The technical tool is the theory of Kolmogorov complexity.
We give the algorithms  and consider the associated predictions.
\end{abstract}
\date{}
\section{Introduction}\label{sect.0}
One can associate the natural numbers 
with a lexicographic length-increasing ordering of finite strings
over a finite alphabet. A natural number corresponds to the string
of which it is the position in this order. Since
a language is a set of sentences (finite strings over a finite alphabet),
it can be viewed as the set of natural numbers.
The learnability of a language under various
computational assumptions is the subject of an immensely influential
approach in \cite{Go65} and especially \cite{Go67}, or the review 
\cite{JORS99}. But surely
in the real world the chance of one sentence of a language 
being used is different
from another one. For example, in general short sentences have a larger
chance of turning up than very long sentences. Thus, the elements
of a given language are distributed in a certain way. There arises
the problem of identifying or approximating this distribution.    

Our model is formulated as follows: 
we are given an infinite sequence of data consisting of elements
drawn from the set (language) according to a certain
probability, and the learner has to identify this probability. 
In general, however much data been encountered, 
there is no point at which the learner can
announce a particular probability as correct with certainty.
Weakening the learning model, the learner might learn to
identify the correct probability in the limit. That is, perhaps the learner
might make a sequence of guesses, finally locking 
on to correct probability and sticking to it
forever---even though the learner can never know for sure that it has
identified the correct probability successfully. 
We shall consider
identification in the limit (following, for example, 
\cite{Go67,JORS99,Pi79}). 
Since this is not enough we additionally restrict
the type of probability.

In conventional statistics, probabilistic models are typically idealized
as having continuous valued parameters; and hence there is an uncountable
number of possible probabilities.
In general it is impossible that a learner 
can make a sequence of guesses
that precisely locks on to the correct values of continuous parameters.
In the realm of algorithmic information theory, in particular in
Solomonoff induction \cite{So64} and here, we reason as follows.
The possible strategies of learners are computable in the sense
of Turing \cite{Tu36}, that is, they are computable functions.
The set of these is discrete and thus countable.
The hypotheses that can be learned are therefore countable, 
and in particular
the set of probabilites from which the learner chooses 
must be {\em computable}.

We consider two cases. In case 1 the data are
drawn independent identically distributed (i.i.d.) from a 
set of natural numbers according to a probability mass function
in a co-c.e. set of computable probability mass functions.
In case 2 the set is finite and the elements of the infinite sequence 
are dependent 
and the data sequence is typical for a measure
from a co-c.e. subset of computable measures.

\subsection{Preliminaries}
Let ${\cal N}$ denote the natural numbers, 
and ${\cal R}$ the real numbers.
We say that we {\em identify} a function
$f$ {\em in the limit} if we have an algorithm which
produces an infinite sequence $f_1, f_2, \ldots$ of functions and 
$f_i=f$ for all but finitely many $i$. This corresponds to
the notion of ``identification in the limit'' in 
\cite{Go67,JORS99,Pi79,ZZ08}.
In this notion at every step
an object is produced and after a 
finite number of steps the target object is
produced at every step. However, we do not know this finite number. It is as
if you ask directions and the answer is ``at the last intersection 
turn right,'' but you do not know which intersection is last.
In the sequel we often ``dovetail''
a computation. This is a technique that interleaves the steps of different
computations ensuring progress of each individual computation. For example,
we have computations $c_1, c_2$. Dovetailing them means 
first performing step 1 of $c_1$, then performing step 2 of $c_1$ followed
by step 1 of $c_2$, then performing step 3 of $c_1$ followed
by step 2 of $c_2$, and so on. 

\subsection{Related work}\label{sect.rel}
In \cite{An88} (citing previous more restricted work)
a target probability mass function was identified in the limit
when the data are drawn i.i.d. in the following setting.
Let the target probability mass function $p$ be an element
of a list  $q_1, q_2, \ldots $ subject to the following
conditions: 
(i) every $q_i : {\cal N} \rightarrow {\cal R}$
is a probability mass function; 
(ii) we exhibit a computable total function $C(i,x,\epsilon)=r$ such that
$q_i(x)-r \leq \epsilon$ with $r,\epsilon >0$ are rational numbers.
That is, there exists a rational number approximation for all probability
mass functions in the list up to arbitrary precision, and we give
a single algorithm which for each such
function exhibits such an approximation. 
The technical means used
are the law of the iterated logarithm and the Kolmogorov-Smirnov test. 
However, the list $q_1, q_2, \ldots $ can not contain 
all computable probability
mass functions because of a diagonal argument, Lemma~\ref{lem.incomp}.

In \cite{BC91} computability questions are apparently ignored.
The {\em Conclusion} 
states ``If the true density [and hence
a probability mass function] is finitely complex [it is computable]
then it is exactly discovered for all sufficiently large sample sizes.''.
The tool that is used is estimation according to 
$\min_q (L(q)+\log(1/\prod_{i=1}^n q(X_i))$. Here $q$ is a probability mass 
function, $L(q)$ is the length of its code and $q(X_i)$ is
the $q$-probability of the $i$th random variable $X_i$. To be able to
minimize over the set of computable $q$'s,
one has to know the $L(q)$'s. 
If the set of candidate distributions is countably infinite, then
we can never know when the minimum is reached---hence at best we have
then identification in the limit.
If $L(q)$ is identified with the Kolmogorov complexity $K(q)$, 
as in Section IV of this reference,
then it is incomputable as already
observed by Kolmogorov in \cite{Ko65} 
(for the plain Kolmogorov complexity; the case of
the prefix Kolmogorov complexity $K(q)$ is the same). 
Computable $L(q)$
(given $q$) cannot be computably enumerated; if they were this
would constitute a computable enumeration of computable $q$'s which 
is impossible by Lemma~\ref{lem.incomp}. To obtain the minimum we
require a computable enumeration of the $L(q)$'s in the estimation formula.
The results hold (contrary to what is claimed in the {\em Conclusion}
of \cite{BC91} and
other parts of the text) not for the set of computable 
probability mass functions since they are not c.e..
The sentence
``you know but you don't know you know'' on the second page of \cite{BC91}
does not hold for an arbitrary
computable mass probability.

In reaction to an earlier version of this paper with too
large claims, in \cite{BMS14} it is shown that it is 
impossible to identify a
computable measure in the limit given an infinite 
sequence of elements from its support which sequence is guarantied
to be typical for some computable measure. 

\subsection{Results}
The set of halting algorithms for computable probabilities (or measures) is 
not c.e., 
Lemma~\ref{lem.incomp} in Appendix~\ref{sect.comput}. 
This complicates the algorithms and
analysis of the results. 
In Section~\ref{sect.1} there is
a computable probability mass function (the target) 
on a set of natural numbers.
We are given an infinite sequence of
elements of this set that are drawn i.i.d., and are 
asked to identify the target.
An algorithm is presented which
identifies the target in the limit almost surely
provided the target is an element of a c.e. or co-c.e. 
set of halting algorithms
for probability mass functions (Theorem~\ref{theo.1}). This
underpins partially the result announced in \cite{HCV11}. 
The technical tool is the strong law of large numbers.
In Section~\ref{sect.3} the set of natural numbers 
is finite and the elements 
of the sequence are allowed 
to be dependent. We are given a guaranty that the sequence is 
typical (Definition~\ref{def.typical})
for at least one measure from a c.e. or co-c.e. set 
of halting algorithms for computable measures.
There is an algorithm which identifies in the limit 
a computable measure for which the data sequence is typical
(Theorem~\ref{theo.3}).
The technical tool is the Martin-L\"of theory of sequential
tests \cite{Ma66} based on Kolmogorov complexity.
In Section~\ref{sect.4} we consider the associated predictions, and in
Section~\ref{sect.concl} we give conclusions.
In Appendix~\ref{sect.comput} we review the used computability notions,
in Appendix~\ref{sect.kolmcomp} we review notions of Kolmogorov complexity,
in Appendix~\ref{sect.measure} we review the used measure and computability
notions. We defer the proofs of the theorems to Appendix~\ref{sect.proofs}.

\section{Computable Probability Mass Functions and I.I.D. Drawing}\label{sect.1}
To approximate a probability in the i.i.d. setting is 
well-known and an easy 
example to illustrate our problem. 
One does this by an algorithm computing the probability
$p(a)$ in the limit for all $a \in L \subseteq {\cal N}$ 
almost surely given the infinite sequence $x_1 ,x_2, \ldots$
of data i.i.d. drawn from $L$ according to $p$. 
Namely, for $n=1,2, \ldots$ for every $a \in L$ occurring 
in $x_1 ,x_2, \ldots , x_n$
set $p_n(a)$ equal to the frequency of occurrences of
$a$ in $x_1,x_2, \ldots ,x_n$.
Note that the different values
of $p_n$ sum to precisely 1 for every $n=1,2, \ldots .$
The output is a sequence $p_1, p_2, \ldots$ of 
probability mass functions such that we have
$\lim_{n \rightarrow \infty} p_n=p$ almost surely, 
by the strong law of large numbers (see Claim~\ref{claim.slln}).
The probability mass functions considered here consist of {\em all}
probability mass functions on $L$---computable or not. 
The probability mass function $p$ is 
represented by an approximation algorithm.
In the limit $p$ is reached almost surely.

Here we deal only with computable probability mass functions.
If $p$ is computable then
it can be represented by a halting algorithm which computes it as 
defined in Appendix~\ref{sect.comput}.
Most known probability mass functions are computable provided
their parameters are computable. 
In order that it is computable
we only require that the probability mass function is finitely
describable and there is a computable process producing it \cite{Tu36}.

One issue is how short the code for $p$ is, a second issue
are the computability properties of the code for $p$, a third
issue is how much of the data sequence is used in the learning
process. The approximation 
of $p$ results in a sequence of codes 
of probabilities $p_1,p_2, \ldots$
which are a list of the sample frequencies in an initial finite
segment of the data sequence. The code length of this list grows
to infinity as the length of the segment grows to infinity. 
The learning process
uses all of the data sequence and the result is an encoding of
the sample frequencies in the data sequence in the limit. 
This holds also if $p$ is computable. 


\begin{theorem}\label{theo.1}
{\sc I.I.D. Computable Probability Identification}
Let $L$ be a set of natural numbers and
$p$ be a probability mass function on $L$ which is
an element of a c.e. or co-c.e. set of halting algorithms
for computable probability mass functions.
There is an algorithm identifying $p$ in the limit almost surely from
an infinite sequence $x_1,x_2, \ldots$ of elements of $L$ drawn i.i.d.
according to $p$. 
The code of $p$ via an appropriate Turing machine
is finite. The learning process uses only a finite initial segment
of the data sequence and takes finite time. 
\end{theorem}
We do not know how large the finite items in the thorem are.
We give an outline of the proof of Theorem~\ref{theo.1}. 
The proof itself is 
deferred to Appendix~\ref{sect.proofs}. We start by extending
the strong law of large numbers to probability mass functions on subsets
of ${\cal N}$. By assumption the target probability mass function $p$
is a member of a c.e. or co-c.e. set of halting algorithms for 
computable probability mass functions listed as list ${\cal A}$. 
If $q$ is in list ${\cal A}$ and $q=p$, then for every 
$\epsilon > 0$ we have $p(a)-q(a) < \epsilon$ for all $a \in L$.
If $q$ is in list ${\cal A}$ and $q \neq p$, then for some $a \in L$
there is a constant $\delta > 0$ such that $|p(a) -q(a)| > \delta$. 
For every $n = 1,2, \ldots$ we estimate
$p(a)$ for all $a \in L$ by the number of 
occurrences of $a$ in the 
$n$-length initial segment of the provided data sequence.

Let $\#a(x_1, \ldots , x_n)$ denote the number of times 
$a=x_i$ ($1 \leq i \leq n$). For $q_i=p$ almost surely
$\lim_{n \rightarrow \infty}
\max_{a \in L}|\#a(x_1, \ldots , x_n)/n-q_k^n(a)| =0$, and for $q_i \neq p$
almost surely $\lim_{n \rightarrow \infty}
\max_{a \in L}|\#a(x_1, \ldots , x_n)/n-q_i^n(a)| > 0$. 
Hence we determine for 
each $n=1,2, \ldots$  the least index $i$ ($1 \leq i \leq n$) in 
the list ${\cal A}$ for which 
$|q_i (a)-\#a(x_1, \ldots , x_n)/n|$ is minimal. This index is called $i_n$.
Let $q_k=p$ with $k$ least.
Eventually the initial $k$-length segment of the list ${\cal A}$ is
co-computably enumerated. Hence
there is a finite $n_0$ such that for all $n \geq n_0$ 
we have $i_n =k$, but we do not know how large $n_0$ is. 
This means that $p$ is identified in the limit.
\begin{remark}
\rm
Since the c.e. and co-c.e. sets strictly contain 
the computable sets, Theorem~\ref{theo.1}
is strictly stronger than the result in \cite{An88} referred to in
Section~\ref{sect.rel}. It is more theoretical but strictly stronger than 
\cite{BC91} that does not give identification in the limit
for classes of computable functions. 

Define the primitive recursive probability mass functions 
as the set of probability mass functions
for which it is decidable that they are constructed from
primitive recursive functions. Since this set is computable it is c.e..
The theorem shows that
identification in the limit is possible for members of this set.
Define the time-bounded probability
mass functions for any fixed computable time bound as
the set of elements for which it is decidable that they are probability
mass functions
satisfying this time bound. Since this set is computable it is c.e..
Again, the theorem shows that
identification in the limit is possible for elements from this set.

Another example is as follows.
Let $L=\{a_1,a_2, \ldots , a_n\}$ be a finite set. The primitive recursive
functions $f_1,f_2, \ldots$ are c.e.. Hence the probability mass functions
$p_1, p_2, \ldots$ on $L$ defined by $p_i(a_j)= f_i(j)/\sum_{h=1}^n f_i(h)$
are also c.e.. Let us call these probability mass functions
simple. By  Theorem~\ref{theo.1} they can be identified
in the limit.
Following the proof of Theorem~\ref{theo.1} in Appendix~\ref{sect.proofs},
we give another example in Example~\ref{exam.A}.
\end{remark}

\section{Computable Measures}\label{sect.3}
As far as the authors are aware, for general measures 
there exist neither an approximation as in 
Section~\ref{sect.1} nor an analog of the strong law of large numbers.
However, there is a notion of typicality of an infinite data sequence
for a computable measure in the Martin-L\"of theory of sequential
tests \cite{Ma66} based on Kolmogorov complexity,
and this is what we use.

Let $L \subseteq {\cal N}$ be finite and
$\mu$ be a measure on $L^{\infty}$ in a co-c.e. 
set of halting algorithms for computable measures. 
In this paper instead of the common notation $\mu(\Gamma_x)$ we use
the simpler notation $\mu(x)$.
We are given a sequence in $L^{\infty}$
which is typical (Definition~\ref{def.typical}) for $\mu$.
Thus, the constituent elements of the sequence are possibly dependent.
The set of typical 
infinite sequences of a computable measure $\mu$ have $\mu$-measure one,
and each typical sequence passes all computable tests for 
$\mu$-randomness in the sense of Martin-L\"of.
This probability model for $L$ is more general than i.i.d. drawing 
according to a probability mass function. It includes stationary processes,
ergodic processes, Markov processes of any order, and other models.

\begin{theorem}\label{theo.3}
{\sc Computable Measure Identification}
Let $L$ be a finite set of natural numbers. We are given
an infinite sequence of elements from $L$ and this sequence
is typical for one measure in a c.e. or co-c.e.
set of halting algorithms for computable measures. 
There is an algorithm which
identifies a computable measure in the limit
for which the sequence is typical.
The code of this measure is an appropriate Turing machine
and finite. The learning process uses only a finite initial segment
of the data sequence. 
\end{theorem}
Let us explain the relation between
Theorem~\ref{theo.1} and Theorem~\ref{theo.3}. 
The set of infinite sequences of i.i.d. draws from a finite set $L$
according to a probability mass function induces
a measure on $L^{\infty}$. Such a measure is called an i.i.d.
measure. The set of computable i.i.d.
measures on $L$ is a proper subset of the set of computable measures on $L$. 
An infinite sequence $x_1,x_2, \ldots$ drawn i.i.d. according to a 
computable probability mass function $p$ on $L$ 
is almost surely typical in
the sense of Definition~\ref{def.typical} for
the induced computable i.i.d. measure $\mu_p$,
and every infinite sequence that is typical for $\mu_p$
is in the set of sequences almost surely drawn i.i.d. according to $p$. 
Hence Theorem~\ref{theo.3} restricted to
i.i.d. measures on finite sets 
implies Theorem~\ref{theo.1} and vice versa.

We give an outline of the proof of Theorem~\ref{theo.3}. 
The proof itself is deferred 
to Appendix~\ref{sect.proofs}. 
Lower semicomputable functions are defined in
Appendix~\ref{sect.comput}. 
Let ${\cal B}$ be a list of a c.e. or co-c.e. set of 
halting algorithms for computable measures with
each measure occurring infinitely many times.
For a measure $\mu$ in the list ${\cal B}$ define
\[
\sigma(j)= \log 1/ \mu (x_1 \ldots x_{j} ) - K(x_1 \ldots x_{j}).
\]
By \eqref{eq.A3}, data sequence $x_1, x_2, \ldots$ 
is typical for $\mu$ iff
$\sup_j \sigma(j) = \sigma$ with $\sigma < \infty$.
By assumption there exists a measure in ${\cal B}$ 
for which the data sequence is typical. Let $\mu_h$ be such a measure
Since algorithms for $\mu_h$ occurs infinitely often in the list ${\cal B}$
there is an algorithm $\mu_{h'}$ in the list ${\cal B}$
with $\sigma_{h'}=\sigma_h$ and
$\sigma_h < h'$. Therefore, there exists a measure
$\mu_k$ in ${\cal B}$ for which the data sequence $x_1, x_2, \ldots$ 
is typical and $\sigma_k < k$ with $k$ least.
If for every $n:=1,2, \ldots$ we compute the least index $i$ 
of $\mu_i$ in ${\cal B}$ such that $\mu_i(x_1, \ldots , x_n)<i$, 
then we identify in the limit a computable measure in ${\cal B}$ 
for which the provided data sequence is typical. 
\begin{remark}
\rm
Let the underlying set $L$ be finite.
Define the primitive recursive measures as the set 
for which it is decidable that they are measures constructed from
primitive recursive functions. Since this set is computable it is c.e..
The theorem shows that
identification in the limit is possible for primitive recursive measures.
Define the time-bounded measures for any fixed computable time bound as
the set of elements for which it is decidable that they are measures
satisfying this time bound. Since this set is computable it is c.e..
Again, the theorem shows that
identification in the limit is possible for elements from this set. 

Let $L$ be a finite set of cardinality $l$, and $f_1,f_2, \ldots$
be a c.e. of the primitive recursive functions. C.e. the strings $x \in L^*$
lexicographical length-increasing. Then every string can be viewed
as the integer giving its position in this order. Define
$\mu_i(\epsilon) = f_i(\epsilon)/f_(\epsilon)=1$, and
inductively for $x \in L^*$ and $a \in L$ 
define $\mu_i(xa)= f_i(xa)/ \sum_{a \in L} f_i(xa)$. 
Then $\mu_i(x)=\sum_{a \in L} \mu_i(xa)$ for all $x \in L^*$.
Call the c.e. $\mu_1, \mu_2, \ldots$ the simple measures.
The theorem shows that
identification in the limit is possible for the set of simple measures.
Following the proof of Theorem~\ref{theo.3} 
in Appendix~\ref{sect.proofs} we show 
another example in Example~\ref{exam.B}.
\end{remark}

\section{Prediction}\label{sect.4}
In Section~\ref{sect.1} 
the data are drawn i.i.d. according to a probability
mass function $p$ on the elements of $L$. Given $p$, 
we can predict the probability $p(a|x_1 ,\ldots, x_n)$ that
the next draw results in an element $a$ when the previous draws
resulted in $x_1 , \ldots , x_n$. The resulting measure on 
$L^{\infty}$ is called an i.i.d. measure. 

For general measures as in Section~\ref{sect.3},
allowing dependent data, the situation is quite different.
We can meet the so-called black swan phenomenon of \cite{Po59}.
Let us give a simple example. The data sequence is $a,a, \dots $
is typical (Definition~\ref{def.typical})
for the measure $\mu_1$ defined by
$\mu_1(x)=1$ for every data sequence $x$
consisting of a finite or infinite string of $a$'s
and $\mu_1(x)=0$ otherwise.
But $a,a, \ldots$ is also typical for $\mu_0$ which gives probability 
$\mu_0(x) = \frac{1}{2}$ for every string $x$ either consisting of 
a finite or infinite string of $a$'s, or a fixed number $n$ of 
$a$'s followed by a finite or
infinite string of $b$'s, and 0 otherwise. Then, $\mu_1$ and $\mu_0$
can give different predictions given a sequence of $a$'s.
But given a data sequence consisting initially
of only $a$'s, a sensible algorithm will predict $a$ as the most
likely next symbol.
However, if the initial data sequence consists of $n$ symbols $a$, then
for $\mu_1$ the next symbol will be $a$ with probability 1, 
and for $\mu_0$ the next
symbol is $a$ with probability $\frac{1}{2}$ and $b$ with probability 
$\frac{1}{2}$. Therefore, while the i.i.d. case allows
us to predict reliably, in the dependent case there is in general
no reliable predictor for the next symbol. 
In \cite{BD62} Blackwell and Dubin show that
under certain conditions predictions of two 
measures merge asymptotically
almost surely. 

\section{Conclusion}\label{sect.concl}
Using an infinite sequence of elements from a set of natural numbers,
algorithms are exhibited that identify in the limit 
the probability distribution associated with this set. 
This happens in two cases:
(i) the target distribution is 
a probability mass function (i.i.d. measure) in a c.e. or co-c.e. set
of computable probability mass functions (computable i.i.d. measures) and
the elements of the sequence are drawn i.i.d. according
to this probability (Theorem~\ref{theo.1});
(ii) the underlying set is finite and the infinite sequence 
is possibly dependent and typical for a computable measure in a 
c.e. or co-c.e. set of computable measures 
(Theorem~\ref{theo.3}).

In the i.i.d. case the target computable probability mass function 
is identified in the limit almost surely, 
in the dependent case the target computable measure is
identified in the limit surely---it is
one out of a set of satsfactory candidate computable measures. 
In the i.i.d. case we use the strong law of large numbers. 
For the dependent case we 
use typicality according to the theory 
developed by Martin-L\"of in \cite{Ma66} embedded in theory of
Kolmogorov complexity. The i.i.d. result is actually a corollary of the
dependency result. 

In both the i.i.d. setting and the dependent setting, eventually we guess
an index of the target (or one target out of many possible targets in the 
measure case) and stick to this guess forever. This last guess is correct. 
However, we do not know when the guess becomes permanent.
We use only a finite unknown-length
initial segment of the data sequence. 
The target for which the guess is correct
is described by a an appropriate Turing machine computing
the probability mass function or measure, respectively.

\appendix

\subsection{Computability}\label{sect.comput}
We can interpret a pair of integers such as $(a,b)$ as rational $a/b$.
A real function $f$ with rational argument is \emph{lower semicomputable}
if it is defined by a rational-valued computable function $\phi(x,k)$
 with $x$ a rational number and $k$ a nonnegative integer
such that $\phi(x,k+1) \geq \phi(x,k)$ for every $k$ and
  $\lim_{k \rightarrow \infty} \phi (x,k)=f(x)$.
This means that $f$ 
 can be computably approximated arbitrary close from below
 (see \cite{LV08}, p. 35).  A function $f$ is  \emph{upper semicomputable}
if $-f$ is semicomputable from below.
 If a real function is both lower semicomputable 
and upper semicomputable then it is \emph{computable}.
A function $f$ is a {\em semiprobability mass function} if 
$\sum_x f(x) \leq 1$ and it is a {\em probability} mass function 
if $\sum_x f(x) = 1$. It is customary to write $p(x)$ for
$f(x)$ if the function involved is a semiprobability mass function.

A set $A \subseteq {\cal N}$ is {\em computable enumerable} (c.e.)  when
we can compute a list $a_1,a_2, \ldots $ of which all elements are members 
of $A$. A c.e. set is also
called recursively enumerable (r.e.). A {\em co-c.e.} set 
$B \subseteq {\cal N}$ is a set whose 
complement ${\cal N} \setminus B$ is c.e.. 
If a set is both c.e. and co-c.e. then it is computable. 
The natural numbers above can be indexes.

Let us explain the relation with identification in the limit.
We explain this for the more complicated case of co-c.e. sets. The
case for c.e. sets is similar. 
Consider a  computable enumeration $o_1, o_2, \ldots$  of a set $O$ 
of objects. A co-c.e. set $S$  is a sublist $C$ of $o_1, o_2, \ldots$ 
such that $C = \{o_i : i \in S\}$. 
The members of $C$ are 
the good objects and the members of $O \setminus C$ the bad objects. 
We computably enumerate the bad objects.
We do not know in what order the bad objects are enumerated or repeated; 
however we do know that the remaining items are the good objects.
These good objects with possible repetitions 
form a list ${\cal A}$, a scattered sublist of the original computable 
enumeration of $O$. This list ${\cal A}$ is a co-c.e. set. 
It takes unknown time to 
enumerate each initial segment of ${\cal A}$,
but we are sure this happens eventually. 
Hence to identify the $k$th element in the list ${\cal A}$
while requiring the first $1, \ldots, k-1$ elements requires 
identification in the limit.

It is known that the overwhelming majority of real numbers are not
computable. If a real number
$a$ is lower semicomputable but not computable, 
then we can computably find nonnegative
integers $a_1, a_2, \ldots$ and $b_1, b_2, \ldots$ such that
$a_n/b_n \leq a_{n+1}/b_{n+1}$ and
$\lim_{n \rightarrow \infty} a_n/b_n = a$. If $a$ is the
probability of success in a trial then this gives an example
of a lower semicomputable probabity mass function which is not computable.
Suppose we are concerned with all and only computable probability
mass functions. There are countably many since there are only
countably many computable functions. But can we computably enumerate them?
The following lemma holds even if the functions are rational valued.

\begin{lemma}\label{lem.incomp}
(i) Let $L \subseteq {\cal N}$ and infinite.
The computable probability mass functions 
on $L$ are not c.e.. 

(ii) Let $L \subseteq {\cal N}$, finite, and $|L|\geq 2$. 
The computable measures on $L$ are not c.e..
\end{lemma}
\begin{proof}
(i) Assume to the contrary that the lemma is false 
and the computable enumeration is
$p_1,p_2, \ldots$.
Compute a probability mass function $p$
with $p(a) \neq p_i(a_i)$ for $a_i \in L$ is the $i$th element of $L$ 
As follows. If $i$ is odd then 
$p(a_i) := f_i(a_i)+f_i(a_i)f_{i+1}(a_{i+1})$
and $p(a_{i+1}) := f_{i+1}(a_{i+1})-f_i(a_i)f_{i+1}(a_{i+1})$,
By construction $p$ is a computable probability
mass function but
different from any $p_i$ in the enumeration $p_1, p_2, \ldots$.

(ii) Since $L$ is finite the set $L^*$ is c.e..
Hence the set of cylinders in $L^{\infty}$ is c.e..
Therefore (ii) reduces to (i). 
\end{proof}

\subsection{Kolmogorov Complexity}\label{sect.kolmcomp}
We need the theory of Kolmogorov complexity \cite{LV08} (originally
in \cite{Ko65} and the prefix version we use here in \cite{Le74}). 
A prefix Turing machine is is a Turing machine with a one-way 
read-only input
tape with an distinguished tape cell called the {\em origin}, 
a finite number of two-way read-write working tapes 
on which the computation takes place, an auxiliary tape on
which the auxiliary string $y \in \{0,1\}^*$ is written, and
a one-way write-only output tape. 
At the start of the computation the input tape
is infinitely inscribed from the origin onwards, and the input head
is on the origin. The machine operates with a binary alphabet.
If the machine halts then the input head has scanned
a segment of the input tape from the origin onwards. We
call this initial segment the {\em program}.  

For every auxiliary $y \in \{0,1\}^*$, 
the set of programs is a prefix code: no program is a proper
prefix of any other program. 
Consider a standard enumeration of all prefix Turing
machines 
\[
T_1, T_2, \ldots . 
\]
Let $U$ denote a 
prefix Turing machine such that for every $z,y\in\{0,1\}^{*}$ and $i\geq1$
we have $U(i,z,y)=T_{i}(z,y)$. That is, for each finite binary program
$z$, auxiliary $y$, and machine index $i\geq1$, 
we have that $U$'s execution
on inputs $i$ and $z,y$ results in the same output as that obtained
by executing $T_{i}$ on input $z,y$. We call such a $U$ a {\em universal}
prefix Turing machine. 

However, there are more ways a prefix Turing machine
can simulate other prefix Turing machines. For example,
let $U'$ be such that $U'(i,zz,y)=T_i(z,y)$ for all $i$ and $z,y$,
and $U'(p)=0$ for $p$ is not $i,zz,y$ for some $i$ and $z,y$.
Then $U'$ is universal also. To distinguish machines like $U$ from
other universal machines, Kolmogorov \cite{Ko65} called machines like $U$
{\em optimal}. 

Fix an optimal machine, say $U$. Define
the conditional {\em prefix Kolmogorov complexity} $K(x|y)$ 
for all $x,y \in \{0,1\}^*$ by 
$K(x|y)=\min_p \{ |p|: p\in\{0,1\}^{*}\:\textrm{and}\: U(p,y)=x\}$.
For the same $U$, define the {\em time-bounded conditional 
prefix Kolmogorov complexity}
$K^t(x|y)=\min_p\{ |p|:\, p\in\{0,1\}^{*}\:
\textrm{and}\; U(p,y)=x\;\textrm{in $t$ steps}\} $.
To obtain the unconditional versions of the prefix Kolmogorov complexities
set $y = \lambda$ where $\lambda$ is the {\em empty} word
(the word with no letters).
It can be shown that $K(x|y)$ is incomputable \cite{Ko65}. 
Clearly $K^t(x|y)$
is computable if $t < \infty$. Moreover,
$K^{t'} (x|y) \leq K^t(x|y)$ for every $t' \geq t$, and
$\lim_{t \rightarrow \infty} K^t(x|y) = K(x|y)$.

\subsection{Measures, Semimeasures, and Computability}\label{sect.measure}
Let $L \subseteq {\cal N}$ and finite.
Given a finite sequence $x=x_1, x_2, \ldots , x_n$ of elements
of $L$, we consider the set of
infinite sequences starting with $x$. The set of all such sequences is
written as $\Gamma_x$, the {\em cylinder} of $x$. 
We associate a
probability $\mu (\Gamma_x)$ with the event 
that an element of $\Gamma_x$ occurs.
Here we simplify the notation $\mu (\Gamma_x)$ and write $\mu (x)$.
The transitive closure of the intersection, complement, 
and countable union of cylinders gives a set of subsets of $L^{\infty}$.
The probabilities associated with these subsets are derived from the
probabilities of the cylinders in standard ways \cite{Ko33a}.
A {\em semimeasure} $\mu$ satisfies the following:
\begin{eqnarray}\label{eq.me}
&& \mu (\epsilon) \leq 1 
\\&& \mu (x) \geq \sum_{a \in L} \mu(xa), 
\nonumber
\end{eqnarray}
and if equality holds instead of each inequality 
we call $\mu$ a {\em measure}.
Using the above  notation, 
a semimeasure $\mu$ 
is {\em lower semicomputable} 
if it is defined by a rational-valued  computable function $\phi(x,k)$
 with $x \in L^*$ 
and $k$ a nonnegative integer
such that $\phi(x,k+1) \geq \phi(x,k)$ for every $k$ and
  $\lim_{k \rightarrow \infty} \phi (x,k)=\mu(x)$.
This means
  that $\mu$ can be computably approximated arbitrary close from below
for each argument $x \in L^*$.

Let $x_1 ,x_2 , \ldots$ be an infinite sequence of elements of
$L$. The sequence is typical
for a computable measure $\mu$ if it passes 
all computable sequential tests (known and unknown alike)
for randomness with respect to $\mu$ in the sense of
Martin-L\"of \cite{Ma66}. One of the highlights of the theory 
of Martin-L\"of is that the sequence passes all these tests iff it passes
a single universal test, \cite{LV08} Corollary 4.5.2 on p 315,
see also \cite{Ma66}.
\begin{definition}\label{def.typical}
\rm
Let $x_1 ,x_2 , \ldots$ be an infinite sequence of elements of
$L \subseteq {\cal N}$ with $L$ finite. The sequence is {\em typical} or
{\em random} for a computable measure $\mu$ iff 
\begin{equation}\label{eq.A3}
\sup_n \{\log \frac{1}{\mu (x_1 \ldots x_n )} - K(x_1 \ldots x_n )\}
< \infty .  
\end{equation}
\end{definition}
The set of infinite sequences that are typical with respect
to a measure $\mu$ have $\mu$-measure one. 
The theory and properties of such sequences for computable measures
 are extensively treated 
in \cite{LV08} Chapter 4. There the term $K(x_1 \ldots x_n )$ in 
\eqref{eq.A3} is given as $K(x_1 \ldots x_n |\mu)$. However, since 
$\mu$ is computable we have $K(\mu) < \infty$ and therefore 
$K(x_1 \ldots x_n |\mu) \leq K(x_1 \ldots x_n)+O(1)$.

\begin{example}
\rm
Let us elucidate by example the notion
of typicality.
Let $\mu_k$ be a measure defined
by $\mu_k(x_ 1\ldots x_n)=1/k$ for $x_i = a$ for
every $1 \leq i \leq n$ and a fixed 
$a \in \{1, \ldots, k\}$,
and $\mu_k(x_ 1\ldots x_n)=0$ otherwise. Then $K(a \ldots a)$ (a sequence
of $n$ elements $a$) equals 
$K(i,n)+O(1) = O(\log n + \log k)$.
(A sequence of $n$ elements $a$ is described by $n$ in $O(\log n)$
bits and $a$ in $O(\log k)$ bits.) 
By \eqref{eq.A3} we have $\sup_{n \in {\cal N}} \{ \log 1/\mu_k(a \ldots a) 
- K(a \ldots a) \} < \infty$. Therefore the infinite 
sequence $a a \ldots$ is typical
for every $\mu_k$. Similarly, 
the infinite sequence $y_1,y_2, \ldots $ 
is not typical for $\mu_k$ for $y_i \in \{1, \ldots ,k\}$ ($i \geq 1$)
and $y_i \neq y_{i+1}$ for some $i$.
Namely, $\sup_{n \in {\cal N}} \{ 1/\mu_k(y_1 y_2\ldots y_n) - 
K(y_1y_2\ldots y_n) \} = \infty$.
\end{example}
The example shows that an infinite sequence of data can be typical for more
than one measure. Hence our task is not to identify a single 
computable measure according to which the
data sequence was generated as a typical sequence,
but to identify a computable measure that {\em could}
have generated the data sequence as a typical sequence.

\subsection{Proofs of the Theorems}\label{sect.proofs}
\begin{proof} {\sc of Theorem~\ref{theo.1}: I.I.D. Computable Probability 
Identification}.
Let $L \subseteq {\cal N}$, and
$X_1, X_2, \ldots $ be a sequence of mutually independent random variables,
each of which is a copy of a single random
variable $X$ with probability mass function $P(X=a)=p(a)$ for $a \in L$. 
Without loss of generalty $p(a)>0$ for all $a \in L$.
Let $\#a(x_1,x_2, \ldots, x_n)$ denote the number of times 
$x_i =a$ ($1 \leq i \leq n$).
\begin{claim}\label{claim.slln}
If the outcomes of the random variables $X_1,X_2, \ldots$ are 
$x_1,x_2, \ldots ,$
then almost surely for all $a \in L$ we have
\begin{equation}\label{eq.strong}
\lim_{n \rightarrow \infty} \; 
\left( p(a)-\frac{\#a(x_1,x_2, \ldots ,x_n)}{n} \right) = 0.
\end{equation}
\end{claim}
\begin{proof}
The strong law of large numbers (originally in \cite{Ko33})
states that if we perform the same 
experiment a large number of times, then almost surely 
the number of successes divided by the number of trials
goes to the expected value, 
provided the mean exists, see 
the theorem on top of page 260 in \cite{Fe68}. 
To determine the probability of an $a \in L$ we consider the 
random variables $X_a$ with just two outcomes $\{a, \bar{a}\}$. 
This $X_a$ is
a Bernoulli process $(q_a,1-q_a)$ where 
$q_a=p(a)$ is the probability of $a$ and 
$1-q_a= \sum_{b \in L\setminus \{a\}} p(b)$ 
is the probability of $\bar{a}$. If we set
$\bar{a} = \min \; (L \setminus \{a\})$, 
then the mean $\mu_a$ of $X_a$ is 
\[
\mu_a = aq_a+\bar{a}(1-q_a) \leq \max \{a, \bar{a}\} < \infty.
\]
Thus, every $a \in L$ incurs a random
variable $X_a$ with a finite mean.
Therefore, 
$(1/n)\sum_{i=1}^n (X_a)_i$ converges 
almost surely to $q_a$ as $n \rightarrow \infty$.
The claim follows.
\end{proof}

Let ${\cal A}$ be a list of a c.e. or co-c.e. set of algorithms for the
computable probability mass functions.
If $q \in {\cal A}$ and $q = p$ then for every $\epsilon >0$ 
and $a \in L$ holds $p(a)-q(a)< \epsilon$.
By Claim~\ref{claim.slln}, almost surely  
\begin{equation}\label{eq.=0}
\lim_{n \rightarrow \infty} \max_{a \in L} \; 
\left(q_i(a) -\frac{\#a(x_1,x_2, \ldots ,x_n)}{n} \right)
= 0.
\end{equation}
If $q \in {\cal A}$ and $q \neq p$ then 
there is an $a \in L$ and a constant $\delta >0$ 
such that $|p(a)-q(a)| > \delta$. 
Again by Claim~\ref{claim.slln}, almost surely
\begin{equation}\label{eq.>0}
\lim_{n \rightarrow \infty} \; 
\max_{a \in L} \left|q_i(a) -\frac{\#a(x_1,x_2, \ldots ,x_n)}{n}\right|
> \delta .
\end{equation}
In the proof of the strong law of large numbers
it is shown that if we draw $x_1,x_2, \ldots$ i.i.d. from a set 
$L \subseteq {\cal N}$ according to a probability mass function $p$
then almost surely the size of the fluctuations in going to the limit
\eqref{eq.=0} 
satisfies $|np(a) - \#a(x_1,x_2, \ldots ,x_n)|/\sqrt{np(a)p(\bar{a})}
< \sqrt{2 \lambda \lg n}$ for every $\lambda >1$ and $n$ is
large enough for all $a \in L$, see \cite{Fe68} p. 204.
Here $\lg$ denotes the natural logarithm. 
Since $p(a)p(\bar{a}) \leq \frac{1}{4}$ and $\lambda = \sqrt{2}$
suffices we obtain $|p(a) - \#a(x_1,x_2, \ldots ,x_n)/n|
< \sqrt{(\lg n)/n}$ for all but finitely many $n$.

Let $q \in {\cal A}$. 
For $q \neq p$ there is an $a \in L$ such that
by \eqref{eq.>0} and the fluctuations in going to that
limit we have $|q(a) - \#a(x_1,x_2, \ldots ,x_n)/n|
> \delta - \sqrt{(\lg n)/n}$ for all but finitely many $n$. 
Since $\delta >0$ is constant,
we have $2\sqrt{(\lg n)/n} < \delta$ for all but finitely many $n$. Hence
$|q(a) - \#a(x_1,x_2, \ldots ,x_n)/n| > \sqrt{(\lg n)/n}$
for all but finitely many $n$.

Let ${\cal A}=q_1,q_2, \ldots$ and $p=q_k$ with $k$ least.
We give the algorithm  with as output a sequence of indexes 
$i_1,i_2, \ldots$ such that all but finitely
many indexes are $k$.
If $L$ is infinite then the algorithm can only use a finite subset of it.
Hence we need to define this finite subset
and show that the remaining elements can be ignored.
Let $A_n = \{a \in L: 
\#a(x_1,x_2, \ldots ,x_n) > 0\}$. 
In case $a \neq A_n$ then $|q(a) - \#a(x_1,x_2, \ldots ,x_n)/n| =q_i(a)$.
We disregard $q_i(a) < \sqrt{(\lg n)/n}$ as follows. 
Let $L=\{a_1,a_2, \ldots \}$.
For each $q_i$ define the set
$B_{i,n} = \{a_1, \ldots, a_m\}$ with $m$ least such that 
$\sum_{j=m+1}^{\infty} q_i(a_j) = 1-\sum_{j=1}^m q_i(a_j) 
< \sqrt{1/n}$.
Therefore, if $a \in L \setminus B_{i,n}$ then $q_i(a) < \sqrt{1/n}$.
The sets $A_n$ and $B_{i,n}$ are finite for all $n$ and $i$. 
Set $L_{i,n}=A_n \bigcup B_{i,n}$. 
Then for every 
$a \in L$ we have 
$|q_k(a)- \#a(x_1,x_2, \ldots ,x_n)/n | \leq  \sqrt{(\lg n)/n}$
for all but finitely many $n$. For $i \neq k$ there is an 
$a \in L_{k,n}$ but no $a \in L \setminus L_{k,n}$ such that
$|q_i(a)- \#a(x_1,x_2, \ldots ,x_n)/n | >
\sqrt{(\lg n)/n}$ for all but finitely many $n$. This leads to the 
following algorithm:

\begin{tabbing}
{\bf for} \= $n:=1,2, \ldots$\\

\> $I:=\emptyset$\=; {\bf for} $i:=1,2, \ldots, n$\\

\> \>{\bf if}
$\max_{a \in L_{i,n}} |q_i(a) - \#a(x_1,x_2, \ldots ,x_n)/n | 
 < \sqrt{(\lg n)/n}$\\

\>\> {\bf then} $I:=I \bigcup \{i\}$; \\

\> $i_n := \min I$
\end{tabbing}

With probability 1 
for every $i < k$ for all but finitely many $n$ we have $i \not\in I$ 
while $k \in I$ for all but finitely many $n$. 
(Note that for every $n=1,2, \ldots$ the main term in the above algorithm is
computable even if $L$ is infinite.) 
The theorem is proven. 
\end{proof}

\begin{example}\label{exam.A}
\rm
We give an example of a list ${\cal A}$ of a co-c.e. 
set halting algorithms for computable probability mass functions.
This set is large but does not contain all probability mass
functions.
A semiprobability mass function is a function for which the
values sum to at most 1.

First we obtain a computable co-enumeration of computable total functions 
which is not c.e..
Let $f: {\cal N} \rightarrow {\cal N}$ 
be a computable time-bound such as the Ackermann function, a total computable
function growing faster than any primitive recursive function, 
and $\phi_1, \phi_2, \ldots$ a standard 
computable enumeration of all partial computable functions.
Computably enumerate all $\phi_i$ such that $\phi_i(j)$ does not
halt within $f(j)$ steps for all $i,j \geq 1$. Eliminate all those from 
$\phi_1, \phi_2, \ldots .$ The result is a subsequence of the original
computable enumeration, a computable co-enumeration of 
total computable functions $\psi_1, \psi_1, \ldots$ which are time bounded
by $f$.
\begin{claim}\label{claim.allq}
\rm
Given a computable co-enumeration of 
computable total functions, one can exhibit 
a computable total function $\phi(i,x,n)=q_i^n(x)$
such that $\phi(i,x,n) \leq \phi(i,x,n+1)$ and
$\lim_{n \rightarrow \infty} q_i^n(x) = q_i(x)$ iff $q_i$ is
a lower semicomputable semiprobability mass function. 
\end{claim}
\begin{proof}
Let $\psi_1, \psi_1, \ldots$ be as above. Computably change every $\psi$
into an algorithm lower semicomputing a semiprobability mass function
$q$, see the proof of Theorem 4.3.1 in \cite{LV08} 
(originally in \cite{ZL70,Le74}).  
For every $a \in L$ denote the $n$th approximation of $q(a)$ in the lower 
semicomputation of $q(a)$ by $q^n(a)$. 
Therefore we can compute
\begin{equation}\label{eq.Q}
{\cal Q} = q_1 ,q_2 , \ldots ,
\end{equation}
a list containing only algorithms which
lower semicompute semiprobability mass total functions.
Without loss of generality the function lower semicomputed by every
algorithm in ${\cal Q}$ is over the alphabet $L$.
\end{proof}

Let $L=\{a_1,a_2, \ldots \}$.
The semiprobability mass functions $q$ in list ${\cal Q}$
such that there is an $n$ for which
$\sum_{i=1}^nq^n(a_n) < 1-1/n$ can be computably
enumerated. The remaining elements in list ${\cal Q}$
are probability mass functions and they are computably co-enumerated.
The intersection of a two
co-c.e. sets is co-c.e.. We show that the remaining 
lower semicomputable probability mass functions
are computable. A probability mass function
$q$ in list ${\cal Q}$ can be computed as follows: for every $\epsilon >0$
let $n_{\epsilon}$ be least such that 
$\sum_{j=1}^n q^n(a_j) \geq 1-\epsilon$ for all $n \geq n_{\epsilon}$. 
Thus every probability 
mass function in list ${\cal Q}$ is computable and we have an algorithm
to compute it. 
\end{example}

\begin{proof} {\sc of Theorem~\ref{theo.3} Computable Measure Identification}
For the Kolmogorov complexity notions 
see Appendix~\ref{sect.kolmcomp}. For
the theory of semicomputable semimeasures, see 
Appendix~\ref{sect.measure}.
In particular
we use the criterion of Definition~\ref{def.typical} to
show that an infinite sequence is typical 
in Martin-L\"of's sense.  
The given data sequence $x_1,x_2, \ldots$ is, by assumption, typical 
for some computable measure $\mu$ and hence 
satisfies \eqref{eq.A3} with respect to $\mu$.
We stress that the data sequence is possibly $\mu$-typical and 
$\mu'$-typical for different computable measures $\mu$ and $\mu'$. 
Therefore we cannot speak of the single {\em true} computable measure, 
but only of {\em a} computable measure
for which the data is typical. 

Let ${\cal B}$ be a list of halting algorithms for a c.e. or co-c.e. set of
computable measures such that each element occurs infinitely many
times in the list. 

\begin{claim}\label{claim.typical}
\rm
There is an algorithm with as input a list 
${\cal B}= \mu_1,\mu_2, \ldots$
and as output a sequence 
of indexes $i_1,i_2, \ldots$. For every large enough $n$ we have 
$i_n= k$ with $\mu_k$ a computable measure 
for which the data sequence is typical.
\end{claim}
\begin{proof}
Define for $\mu$ in ${\cal B}$ 
\[
\sigma(j)= \log 1/ \mu (x_1 \ldots x_{j} ) - K(x_1 \ldots x_{j}).
\]
Since  $K$ is upper semicomputable and $\mu$ is computable,
the function $\sigma(j)$
is lower semicomputable for each $j$.
Define the $n$th value in the lower semicomputation of 
$\sigma(j)$ as $\sigma^n(j)$. 
By \eqref{eq.A3}, the data sequence $x_1, x_2, \ldots$ 
is typical for $\mu$ if
$\sup_{j\geq 1} \sigma(j) = \sigma < \infty$ 
In this case, since $\mu$ is lower semicomputable,
$\max_{1 \leq j \leq n} \sigma(n) \leq \sigma$ for all $n$.
In contrast, the data sequence is not typical for
$\mu$ if $\sigma(n) \rightarrow \infty$ with $n \rightarrow \infty$
implying $\sigma^n(n) \rightarrow \infty$ with $n \rightarrow \infty$.

By assumption there exists a measure in ${\cal B}$
for which the data sequence is typical. Let $\mu_h$ be such a measure
Since algorithms for $\mu_h$ occur infinitely often in the list ${\cal B}$
there is an algorithm $\mu_{h'}$ in the list ${\cal B}$
with $\sigma_{h'}=\sigma_h$ and
$\sigma_h < h'$. Therefore, there exists a measure
$\mu_k$ in ${\cal B}$ for which the data sequence $x_1, x_2, \ldots$
is typical and $\sigma_k < k$ with $k$ least.
The algorithm to determine $k$ 
is as follows. 

\begin{tabbing}
{\bf for} \= $n:=1,2, \ldots$\\ 

\> {\bf if} $i \leq n$ is least such that  
$\max_{1 \leq j \leq n}\sigma_i^n(j) <i$\\

\>  {\bf then} 
output $i_n = i$ {\bf else} output $i_n=1$.
\end{tabbing}

Eventually $\max_{1 \leq j \leq n}\sigma_k^n(j)<k$ for large enough $n$,
and $k$ is the least index of elements in ${\cal B}$ for which this holds.
Hence there exists an $n_0$  such that $i_n=k$ for all $n \geq n_0$.
\end{proof}

For large enough $n$ we have by Claim~\ref{claim.typical} a test
such that we can identify in the limit an index of a measure 
in ${\cal B}$ for which the provided data sequence is typical. 
Hence there is an $n_0$ such that $i_n=k$ for all $n \geq n_0$.
We do not care what $i_1, \ldots , i_{n-1}$ are.
This proves the theorem.
\end{proof}

\begin{example}\label{exam.B}
\rm
We give an example of a list ${\cal B}$ of halting
algorithms for a co-c.e. set of 
computable measures. 
\begin{claim}\label{claim.allm}
\rm
Given a co-enumeration of computable total 
functions, one can exhibit
a computable total function $\phi(i,x,n)=\mu_i^n(x)$
such that $\phi(i,x,n) \leq \phi(i,x,n+1)$ and
$\lim_{n \rightarrow \infty} \mu_i^n(x) = \mu_i(x)$.
\end{claim}
\begin{proof}
To eliminate functions with undefined values,
let $\psi_1, \psi_1, \ldots$ be a co-enumeration of total functions in
a standard computable enumeration of all partial computable
functions as in Example~\ref{exam.A}. 
Computably change every $\psi$
into an algorithm lower semicomputing a semimeasure $\mu$,
similar to the method in the proof of Theorem 4.5.1 of 
\cite{LV08} pp. 295--296 (originally in \cite{ZL70}).
For every $x \in L^*$ denote the $n$th approximation of 
$\mu(x)$ in the lower semicomputation of $\mu(x)$ by $\mu^n(x)$.
Therefore we can compute
\begin{equation}\label{eq.M}
{\cal M} = \mu_1 ,\mu_2 , \ldots ,
\end{equation}
a list containing only algorithms which
lower semicompute semimeasures.
Without loss of generality the function lower semicomputed by every
algorithm in ${\cal M}$ is over the alphabet $L$.
\end{proof}

Every function in the list will be in the list infinitely often,
which follows simply from the fact that there are infinitely many
algorithms which lower semicompute a given function.
It is important to realize that, 
although the code of a computable
measure may be in list ${\cal M}$,
it is there as an algorithm lower semicomputing the measure.
By Claim~\ref{claim.allm} we can co-enumerate 
halting algorithms that lower semicompute
semimeasures \eqref{eq.M}.  
Let $\mu^n(x)$ denote the $n$th lower
semicomputation of $\mu(x)$ for a semimeasure $\mu$.
The semimeasures $\mu$ in list ${\cal M}$ such that there are
$x \in L^*$ and $n < \infty$ such that either $\mu^n(\epsilon)<1-1/n$ 
or $\mu^n(x)-\sum_{a \in L}\mu^n(xa) < 1/n$ can be computably 
enumerated. The remaining elements in list ${\cal M}$
are wide set of computable measures (but not all) and they are co-c.e..
A lower semicomputable algorithm for a measure
can be converted to a computable algorithm. 
To see this, let $L=a_1,a_2, \ldots, a_n$.
Let $\mu$ be a lower semicomputable semimeasure
with $\sum_{a \in L} \mu (xa) = \mu (x)$ for all $x \in L^*$ 
and $\mu (\epsilon)=1$.
Then, we can approximate all $\mu (x)$ to any degree of precision
starting with $\mu (a_1), \mu (a_2), \ldots$ and determining $\mu (x)$
for all $x$ of length $n$, for consecutive $n=1,2, \ldots .$
\end{example}

\section*{Acknowledgement}
We thank Laurent Bienvenu for pointing out an error in the
an earlier version and elucidating comments. Drafts of this paper
proceeded since 2012 in various states of correctness
through arXiv:1208.5003 to arXiv:1311.7385.

\section*{Biographies}
{\sc Paul M.B. Vit\'anyi} received his Ph.D. from the Free University
of Amsterdam (1978). He is a CWI Fellow at
the National Research Institute for Mathematics and Computer
Science in the Netherlands, CWI,
and Professor of Computer Science
at the University of Amsterdam.  He served on the editorial boards
of Distributed Computing, Information Processing Letters,
Theory of Computing Systems, Parallel Processing Letters,
International journal of Foundations of Computer Science,
Entropy, Information,
Journal of Computer and Systems Sciences (guest editor),
and elsewhere. He has worked on cellular automata,
computational complexity, distributed and parallel computing,
machine learning and prediction, physics of computation,
Kolmogorov complexity, information theory, quantum computing, publishing 
more than 200 research papers and some books. He received a Knighthood
(Ridder in de Orde van de Nederlandse Leeuw) and is member of the
Academia Europaea. Together with Ming Li
they pioneered applications of Kolmogorov complexity
and co-authored ``An Introduction to Kolmogorov Complexity
and its Applications,'' Springer-Verlag, New York, 1993 (3rd Edition 2008),
parts of which have been translated into Chinese,  Russian and Japanese.
Web page: http://www.cwi.nl/~paulv/

{\sc Nick Chater} received his Ph.D from the 
University of Edinburgh (1990). He is
Professor and Head of the Behavioural Science Group at Warwick Business School
and has served as an Associate Editor for Management Science, Psychological
Science, Psychological Review and Cognitive Science. His research focusses on
the scope and limits of rational models of cognition, and has published on
language acquisition, processing and evolution; reasoning and decision making;
perception; and similarity and categorization. He has over 200 research
publications, including several books, and  has won four national awards for
psychological research, and has served as Associate Editor for the journals
Cognitive Science, Psychological Review, and Psychological Science. He was
elected a Fellow of the Cognitive Science Society in 2010 and a Fellow of the
British Academy in 2012. He has been awarded a European Research Council
Advanced Grant for the period 2012--2017, on the ``Cognitive and social
foundations of rationality.'' Web page:
http://www2.warwick.ac.uk/fac/soc/wbs/subjects/bsci/people/nickchater/

\end{document}